\documentclass{isipta2025}

\usepackage{tikz} 
\usetikzlibrary{positioning}
\usepackage{booktabs} 
\usepackage{fvextra} 
\usepackage{array} 
\usepackage{tabularx} 
\usepackage{graphicx}
\usepackage{nicefrac}

\usepackage{pgfplots}
\usepackage{subfigure}
\usepackage{array}
\usepackage{framed,color}
\newcolumntype{C}[1]{>{\centering\arraybackslash}p{#1}}

\usetikzlibrary{arrows,automata,pgfplots.colorbrewer,}
\usepgfplotslibrary{external}
\pgfplotsset{
    compat=1.18,
    cycle from colormap manual style/.style={
        cycle list={
            [of colormap=Paired]
        }
    },
}

\newcommand{\class}{c}
\newcommand{\Class}{C}
\newcommand{\Classes}{\mathcal{C}}

\newcommand{\prediction}{\hat{\class}}

\newcommand{\classifier}{h}
\newcommand{\classifierp}{h_{\prob}}
\newcommand{\classifierpa}{h_{\proba}}

\newcommand{\feature}[1]{f_{#1}}
\newcommand{\Feature}[1]{F_{#1}}
\newcommand{\FeatureSet}[1]{\mathcal{F}_{#1}}
\newcommand{\features}{f}
\newcommand{\Features}{F}
\newcommand{\FeaturesSet}{\mathcal{F}}
\newcommand{\numfeatures}{N}

\newcommand{\prob}{P}
\newcommand{\proba}{p}

\newcommand{\pertproba}{p'}
\newcommand{\probfamily}{\mathcal{P}}

\newcommand{\testprob}{\prob_{\text{test}}}
\newcommand{\testproba}{\proba_{\text{test}}}
\newcommand{\testset}{D_{\text{test}}}
\newcommand{\trainprob}{\prob_{\text{train}}}

\newcommand{\trainset}{D_{\text{train}}}

\newcommand{\maxprob}{u_{\text{m}}}

\newcommand{\entropy}{u_H}
\newcommand{\aleatunc}{u_{a}}
\newcommand{\totunc}{u_{t}}
\newcommand{\epistunc}{u_{e}}
\newcommand{\Nmodels}{M}

\newcommand{\robmetric}{\varepsilon}
\newcommand{\globeps}{\varepsilon_{\mathrm{glob}}}
\newcommand{\loceps}{\varepsilon_{\mathrm{loc}}}

\newcommand{\globperta}{\probfamily^{\,\mathrm{glob}}_{\proba,\,\robmetric}}

\newcommand{\globpertzeroa}{\probfamily^{\,\mathrm{glob}}_{\proba,\,0}}

\newcommand{\locperta}{\probfamily^{\,\mathrm{loc}}_{\proba,\,\robmetric}}

\newcommand{\locpertzeroa}{\probfamily^{\,\mathrm{loc}}_{\proba,\,0}}

\newcommand{\bulletprobfamily}{\probfamily_{\hspace{-1pt}\bullet}}

\newcommand{\bulletprobfamilyPa}{\probfamily_{\smash{\proba},\bullet}}

\newcommand{\probmeasfamily}[1]{\mathcal{M}_{#1}}

\definecolor{global_eps}{rgb}{0.2, 0.72, 0.92}
\definecolor{local_eps}{rgb}{0.62, 0.82, 0.2}
\definecolor{aleatoric}{rgb}{0.78, 0.28, 0.38}
\definecolor{entropies}{rgb}{0.38, 0.506, 0.883}
\definecolor{max_prob}{rgb}{0.99, 0.68, 0.72}
\definecolor{total}{rgb}{0.29, 0.53, 0.13}
\definecolor{epistemic}{rgb}{0.43, 0.5, 0.5}

\title{Robustness quantification: a new method for assessing the reliability of the predictions of a classifier}

\author[]{Adri\'an Detavernier}
\author[]{Jasper De Bock}

\affil[]{Foundations Lab for imprecise probabilities\\Ghent University\\Belgium}

\begin{document}

\maketitle

\begin{abstract}
Based on existing ideas in the field of imprecise probabilities, we present a new approach for assessing the reliability of the individual predictions of a generative probabilistic classifier. We call this approach robustness quantification, compare it to uncertainty quantification, and demonstrate that it continues to work well even for classifiers that are learned from small training sets that are sampled from a shifted distribution.
\end{abstract}


\section{Introduction}
\label{sec:intro}

Let's say that you start using an AI model for a high-risk application, such as medical diagnosis or self-driving cars.
At that point, if you use a classifier to automatically make predictions, it does not suffice that your model generally has a \(99\%\) accuracy: you also want to know whether the current prediction is part of that other \(1\%\) of the cases.
That is, you'd really want to know how reliable the current prediction is.

Several approaches for quantifying this reliability exist, two of which we intend to compare. The most common approach is uncertainty quantification \cite{hullermeier2021aleatoric,sale2024label,pmlr-v180-hullermeier22a}, whose aim it is to numerically quantify the amount of uncertainty associated with a prediction. We instead focus on an alternative approach, which we call robustness quantification, whose aim it is to quantify how robust a prediction is against uncertainty. One might say that we quantify how much uncertainty the model \emph{could} cope with before changing its prediction, and this regardless of how much uncertainty is actually present.
While our proposed terminology is new, the ideas behind robustness quantification have been successfully tested several times by now \cite{NIPS2014_09662890,correia2020robustclassificationdeepgenerative,pmlr-v62-mauá17a}, always relying on techniques from imprecise probabilities \cite{augustin2014introduction}. Our contribution consists in conceptually introducing this approach in a general setting and comparing it with uncertainty quantification.

In particular, we focus on how uncertainty and robustness quantification compare in cases where the available data is limited or when there is a distribution shift between the train and test data.
Our motivation for this comparison is fueled by the fact that these problems arise frequently in practice, and can have a big impact on how a classifier performs in the real world (on unseen data) \cite{quinonero2022dataset,pmlr-v139-koh21a,raudys1991small,vabalas2019machine}.
For now, we mainly focus on the naive Bayes classifier as a test case because of its simplicity and efficiency.
However, our methods can be applied to more complex classifiers as well.

From our experiments we conclude that robustness quantification keeps performing well even with limited data or in the presence of a distribution shift, while uncertainty quantification sees a decrease in performance.
Moreover, our robustness metrics keep on performing well even without distribution shift and with enough data, meaning that there is no trade-off paid for the better protection that they offer.

\section{Classification}
\label{sec:classification}

We consider the problem of classification, where the goal is to predict the correct class of an instance based on its features.
An example of a high-risk classification problem is predicting whether a patient has cancer (the class) based on its medical data and images (the features).

\subsection{Classifiers}
We denote the class variable by \(\Class\), which can take on any value \(\class\) in the finite set of classes \(\Classes\).
An instance can have one or several features \(\Feature{i}\), with \(i \in \{1, \dots, \numfeatures\}\), where \(\numfeatures\) is the number of features.
Every one of these features takes values \(\feature{i}\) in the finite set \(\FeatureSet{i}\).
The vector containing all the feature variables \(\Feature{i}\) is denoted as \(\Features = (\Feature{1}, \dots, \Feature{\numfeatures})\).
The particular feature values \(\feature{i}\) of an individual instance are collected in a feature vector \(\features = (\feature{1}, \dots, \feature{\numfeatures})\), to which we often refer as the (set of) features of said instance.
All such possible feature vectors are collected in the set \(\FeaturesSet = \FeatureSet{1} \times \dots \times \FeatureSet{\numfeatures}\).
Every instance is determined by the combination of a class \(\class\) and its features \(\features\).
In practice, however, we typically only know the features of an instance.
The aim of a classifier is then to predict the unknown class, given those features.

Formally, this is done with a classifier:  a deterministic function \(\classifier: \FeaturesSet \to \Classes\) that maps each feature vector \(\features \in \FeaturesSet\) to a class \(\classifier(\features) \in \Classes\).
Ideally, we would want a classifier to always predict the correct class.
This is not possible though because instances with the same set of features may nevertheless have different classes.
We then want a classifier to be correct as often as possible.

In the following sections, we dive deeper into how classifiers, and probabilistic ones in particular, predict the class of an instance, and discuss some of the problems that arise when trying to learn such classifiers from data.

\subsection{Generative probabilistic classifiers}
Rather than predict the class directly, a probabilistic classifier first predicts the (conditional) probability \(\proba(\class \vert \features)\) of every possible class \(\class\) given the features \(\features\) of the considered instance, and then predicts the class of this instance based on these probabilities.
This predicted class is usually taken to be the one that has the highest probability given the features:
\begin{equation}
	\label{eq:prediction}
	\classifier(\features) \in \arg \max_{\class \in \Classes} \proba(\class \vert \features).
\end{equation}
In most cases, this inclusion is simply an equality.
However, it is possible, although very unlikely in practice, that there are several equiprobable classes that all have the highest probability.
The set \(\arg \max_{\class \in \Classes} \proba(\class \vert \features)\) then contains these classes, and \(\classifier(\features)\) is then arbitrarily taken to be one of them.
The chosen or predicted class \(\classifier(\features)\) is called the \emph{prediction}.
If the feature vector and classifier are clear from the context, we will sometimes denote the predicted class \(\classifier(\features)\) by \(\prediction\) as well.

A probabilistic classifier can predict the (conditional) probabilities of the possible classes given the features in, essentially, two ways, that each correspond to a type of probabilistic classifier: discriminative or generative.
Discriminative classifiers try to determine the conditional probabilities \(\proba(\class\vert\features)\) directly.
We instead focus on generative classifiers, which first determine a joint probability distribution \(\prob\) for the class and features.
Since $\Classes$ and $\FeaturesSet$ are finite in this contribution, every such distribution \(\prob\) is uniquely characterized by its (probability) mass function \(\proba \colon \Classes \times \FeaturesSet \to [0,1]\), which simply assigns a probability \(\proba(\class, \features) \coloneq \prob(\Class = \class \text{ and } \Features = \features)\) to all $\class\in\Classes$ and $\features\in\FeaturesSet$.
The conditional probabilities that are used to classify instances are then obtained by conditioning: \(\proba(\class\vert\features) = \nicefrac{\proba(\class, \features)}{\proba(\features)}\), assuming that the observed features \(\features\) have a non-zero (marginal) probability \(\proba(\features)\coloneq \sum_{\class \in \Classes} \proba(\class, \features)>0\).
To emphasize the dependency of a generative classifier on the mass function \(\proba\), we make this explicit in our notation by denoting it as \(\classifierpa\).
Our reason for focussing on generative classifiers is that our ideas about robustness, and the theory of imprecise probabilities on which they are based, can be applied more naturally in this setting.

Since we maximize over the possible classes to get the prediction \(\classifierpa(\features)\) for a feature vector \(\features\), the denominator $\proba(\features)$ in Bayes's rule---which does not depend on the class---can be ignored in Equation \eqref{eq:prediction}.
For generative classifiers, Equation \eqref{eq:prediction} therefore simplifies as follows:
\begin{equation}
	\label{eq:joint_prediction}
	\classifierpa(\features) \in \arg \max_{\class \in \Classes} \proba(\class, \features).
\end{equation}
If \(\proba(\features)=0\), then \(\proba(\class, \features)=0\) for all \(\class\) as well, in which case Equation \eqref{eq:joint_prediction} tells us that \(\classifierpa(\features)\) is an arbitrary class in \(\Classes\); this case is typically avoided in practice though.

\subsection{Learning a classifier}
If we want to use a (generative) classifier, we first need to somehow learn its joint distribution \(\prob\).
This is typically done with a labeled data set, where the class of each instance is known.
We will call this data set the training set, and denote it by \(\trainset\).
In practice, it is typically assumed that the training set has an underlying distribution that generated it: the training distribution \(\trainprob\).
Once a classifier has been learned, we of course want to put it to use on unseen data to see how well it performs.
To that end, we consider a second data set \(\testset\) that we call the test set.
Here too, it is typically assumed that this data set has an underlying distribution, called the test distribution, which we denote by \(\testprob\).
In a perfect scenario then, the joint distribution $\prob$ of our classifier is equal to $\testprob$, yielding a classifier that issues the correct prediction as often as possible, at least on average for large test sets \(\testset\). However, this almost never occurs because of, essentially, two reasons. One reason is that $\testprob$ may not be equal to $\trainprob$, which is what we will come to next, but let us first look at the case where they are equal.

If $\testprob=\trainprob$, then the task of learning a generative classifier basically boils down to trying to make sure that $\prob$ approximates $\trainprob$ sufficiently well, the perfect scenario being that $\prob=\trainprob=\testprob$.
However, this perfect scenario is only possible in the idealized situation where \(\trainset\) is infinitely large.
For (very) small training sets \(\trainset\), in fact, it often happens that \(\trainset\) is not at all representative for $\trainprob$, making it impossible for $\prob$  to approximate \(\trainprob\) with any reasonably accuracy.
A possible solution to this problem consists in collecting more data.
Unfortunately, this not always possible.

Regardless of whether $\trainset$ is large enough to be able to accurately learn $\trainprob$, there is also the additional problem that $\trainprob$ may not be equal to $\testprob$, a phenomenon known as \emph{distribution shift}. Consider for example a situation where a classifier is trained based on medical data from one hospital, and then used in another hospital \cite{zech2018variable}, or where it is trained on colored images but then used on black-and-white images.
In such cases, even if $\prob=\trainprob$, the performance of the corresponding classifier $\classifierp$ is bound to suffer.
This is a very common problem, that is known to have a big impact on the performance of classifiers \cite{quinonero2022dataset,pmlr-v139-koh21a}.
Distribution shift is often ignored though, by (tacitly) assuming that $\trainprob$ and $\testprob$ are equal.

Regardless of how well $\prob$ approximates $\testprob$, the performance of a classifier is often measured in terms of \emph{accuracy}: the number of correctly predicted instances  divided by the total number of predictions.
This metric is, in most cases, good enough to evaluate a classifier, since the goal is often to be correct as much as possible.
However, in cases where a bad prediction could have huge consequences, we are no longer only interested in the overall performance of the classifier.
In such cases, it is equally important to assess the reliability of each individual prediction, allowing us to flag those that might be unreliable.
This brings us to reliability quantification.

\section{Reliability quantification}
\label{sec:reliability}

In high-risk applications, it is of crucial importance to know for each separate prediction whether it is correct or not, because a wrong prediction could have severe consequences.
Although it is of course infeasible to construct a classifier that is always correct, we could try to construct a classifier that is correct in most cases, and can single out the other, more difficult cases.
This is exactly what we want to achieve with \emph{reliability quantification}.
That is, we want to quantify the reliability of each individual prediction.
Any numeric value that aims to quantify (some aspect of) the reliability of a prediction, we call a \emph{reliability metric}.
To stress that these reliability metrics are associated with individual predictions, instead of the overall performance of a classifier,  we say that they are \emph{instance-based}.
We restrict ourselves to two methods for quantifying instance-based reliability, each of which focuses on a different aspect of reliability: \emph{uncertainty quantification} and \emph{robustness quantification}.

\subsection{Uncertainty}
Uncertainty quantification, as the terminology suggests, aims to quantify the uncertainty associated with individual predictions, the idea being that uncertainty negatively influences reliability.
To better understand some of the challenges of this approach, it is helpful to distinguish two different types of uncertainty that are often considered in this context \cite{hullermeier2021aleatoric}: \emph{aleatoric} and \emph{epistemic} uncertainty.

Aleatoric uncertainty has to do with the intrinsic variability of the test data, as captured by \(\testprob\).
For an instance with features $\features$, it is the uncertainty that corresponds to the conditional probabilities $\testproba(\class\vert\features)$: even if we know $\testprob$ perfectly, then still, several classes $\class$ remain possible, each with their own probability of occurrence $\testproba(\class\vert\features)$.
Since it is intrinsic to the process, aleatoric uncertainty is irreducible, meaning that it will always remain no matter how hard we try; it is the reason why in practice no classifier can be expected to always yield a correct prediction.


Epistemic uncertainty, on the other hand, has to do with the fact that we typically do not know \(\testprob\), in the sense that the joint distribution $\prob$ of a classifier will typically not be equal to $\testprob$.
One aspect of this epistemic uncertainty has to do with the fact that $\prob\neq\trainprob$, due to modelling choices or the limited availability of training data; this could be addressed by collecting more data, and is therefore considered reducible.
Another aspect of epistemic uncertainty has to do with the presence of distribution shift ($\trainprob\neq\testprob$), which can only be reduced if we have access to labelled data from $\testprob$.

The aim of uncertainty quantification, then, is to numerically quantify these two types of uncertainty with uncertainty metrics.
This is extremely challenging though.
On the one hand, for epistemic uncertainty, there is typically no way of knowing to which extent $\trainset$ is representative for $\trainprob$, nor how much $\testprob$ differs from $\trainprob$, making it very difficult to assess to which degree $\prob$ differs from $\testprob$.
On the other hand, uncertainty metrics that aim to quantify aleatoric uncertainty would ideally be based on the conditional probabilities $\testproba(\class\vert\features)$.
In practice, however, \(\testprob\) is not available, and one then has to do with the estimated conditional probabilities $\proba(\class\vert\features)$ instead, which may differ considerably from the real ones $\testproba(\class\vert\features)$ due to the presence of epistemic uncertainty.

Nevertheless, many \emph{uncertainty metrics} have been developed, some of which we discuss in more detail in Section~\ref{sec:uncertainty}.
Some of these metrics try to quantify aleatoric and epistemic uncertainty separately, while others try to quantify the combined effect of both, to which they refer as total uncertainty. 
However, one might question to which extent it is possible to make these distinctions while quantifying uncertainty though, since both notions are always intertwined: as we have just explained, aleatoric uncertainty cannot be reliably quantified without quantifying epistemic uncertainty, and epistemic uncertainty itself is very hard to quantify.
It is therefore to be expected that these methods will perform significantly worse in the presence of epistemic uncertainty, as we will also demonstrate in our experiments.

\subsection{Robustness}
Motivated by these difficulties, we here put forward a different approach, called robustness quantification. The main idea is to approach reliability from a different angle: instead of quantifying its uncertainty, which is notoriously hard to do, we instead quantify the robustness of a prediction. That is, we quantify how much epistemic uncertainty the classifier could handle without changing its prediction, thereby completely avoiding the difficult task of quantifying how much epistemic uncertainty there actually is. A prediction is then deemed reliable if it is robust, meaning that it would remain true even in the face of severe epistemic uncertainty.
We formalize this concept in Section~\ref{sec:robustness} for general generative classifiers, and further develop it in Section~\ref{sec:naive_classification} for the particular case of Naive Bayes classifiers.

The main feature that sets our approach apart from other approaches that consider the robustness of a classifier, such as adversarial robustness~\cite{carlini2019evaluating,bai2021recent} or robustness against distribution shift~\cite{NEURIPS2020_d8330f85}, is its instance-based character: we consider the robustness of individual predictions, whereas these other approaches consider the robustness of the overall performance of a classifier. For adversarial robustness~\cite{carlini2019evaluating,bai2021recent}, another difference is that we consider robustness against epistemic uncertainty (resulting from, for example, distribution shift), instead of robustness against manipulations of the (in that case typically continuous) features.


\section{Uncertainty quantification}
\label{sec:uncertainty}

Before we delve into the topic of robustness quantification, and robustness metrics in particular, let us first take a brief look at a number of uncertainty metrics.
The first two metrics single out a specific aspect of the uncertainty about the class given the features: the probability of the predicted class and the (Shannon) entropy, respectively.
The last three metrics are also based on entropy, but combine this with ensemble techniques to better estimate the uncertainty \cite{shaker2020aleatoric}. Each of these five uncertainty metrics can be applied to arbitrary generative classifiers. For a more extensive overview, which also covers metrics for discriminative classifiers, or for specific model architectures, we refer the interested reader to the work of \citet{hullermeier2021aleatoric}.

\paragraph{Maximum probability}
A first straightforward uncertainty metric, denoted by \(\maxprob\), is one minus the conditional probability of the predicted class \(\prediction\), or equivalently, one minus the maximum probability over all classes:
\begin{equation}
	\maxprob(\features) = 1-\max_{\class \in \Classes} \proba(\class \vert \features) = 1-\proba(\prediction \vert \features).
\end{equation}
In an idealized situation where there is no epistemic uncertainty (so $\prob=\testprob$), this is the probability of making an incorrect prediction, and hence a perfect metric for quantifying aleatoric uncertainty. In general, in the presence of epistemic uncertainty, this might then be taken to represent a combination of both types.

\paragraph{Entropy}
A second metric for quantifying uncertainty, which is quite popular and has its roots in information theory, is \emph{entropy}. In particular, we consider the entropy of the conditional probability mass function for the class given the features \(\features\):
\begin{equation}
	\entropy(\features) = - \sum_{\class \in \Classes} \proba(\class \vert \features) \log_2 \proba(\class \vert \features).
\end{equation}
Entropy mostly tells us something about the shape of the conditional probability mass function: the more uniform it is, the higher the entropy.
Since the uniform case corresponds to randomly guessing the correct class, it makes sense to associate high values for $\entropy(\features)$ with there being more uncertainty; in particular, it is typically taken to be a metric that aims to capture the total uncertainty~\cite{shaker2020aleatoric}.


The last three uncertainty metrics combine entropy with ensemble techniques to quantify aleatoric, total and epistemic uncertainty, respectively.
Instead of training the classifier once on the training set $\trainset$, we now construct several bootstrap samples (obtained by sampling from $\trainset$ set with replacement to obtain a data set of the same size as $\trainset$) and each of these samples train a new classifier.
The number of bootstrap samples, and thus the number of classifiers we train, is denoted by \(\Nmodels\).
We denote the predicted joint probability mass functions of these different classifiers by \(\proba_i\), with \(i \in \{1,\dots,\Nmodels\}\). In our experiments in Section~\ref{sec:experiments}, we use $M=10$.

\paragraph{Aleatoric uncertainty}
The \emph{aleatoric uncertainty} is then estimated as follows:
\begin{equation}
	\aleatunc(\features) = - \frac{1}{\Nmodels} \sum_{i=1}^{\Nmodels} \sum_{\class \in \Classes} \proba_i(\class \vert \features) \log_2 \proba_i(\class \vert \features).
\end{equation}
This is the average entropy of the predicted conditional probability distributions over all the classifiers in the ensemble.
The reason this is typically associated with aleatoric uncertainty, is because averaging over the ensembles is believed to reduce the influence of the epistemic uncertainty as much as possible.

\paragraph{Total uncertainty}
The total uncertainty, on the other hand, is then estimated as the entropy of the average predicted conditional probability distribution, given by
\begin{equation}
	\totunc(\features) = - \sum_{\class \in \Classes} \proba_{\mathrm{av}}(\class\vert\features) \log_2 \proba_{\mathrm{av}}(\class\vert\features),
\end{equation}
with $\proba_{\mathrm{av}}(\class\vert\features)\coloneq\nicefrac{1}{\Nmodels} \sum_{m=1}^{\Nmodels} \proba_i(\class \vert \features)$.
So we see that this uncertainty metric is similar to $\entropy$, but with $\proba$ replaced by the average of the different $\proba_i$.

\paragraph{Epistemic uncertainty}
The epistemic uncertainty, finally, is taken to be the difference of the previous two metrics:
\begin{equation}
	\epistunc(\features) = \aleatunc(\features) - \totunc(\features).
\end{equation}
Underlying this metric is an assumption that the total uncertainty is the sum of the aleatoric and epistemic uncertainty.
This assumption seems questionable though, given that both types of uncertainty are intertwined.

\section{Robustness Quantification}
\label{sec:robustness}

To quantify the robustness of a prediction of a classifier, we need to somehow assess how much epistemic uncertainty this classifier can handle before that particular prediction will change. That is, how different $\testprob$ can be from $\prob$ while still providing the same prediction.
Inspired by ideas from the imprecise probabilities literature \cite{NIPS2014_09662890,correia2020robustclassificationdeepgenerative,pmlr-v62-mauá17a}, we will do this by artificially perturbing (the mass function $\proba$ associated with) the joint distribution \(\prob\), defining a robustness metric as the minimal perturbation for which the prediction is no longer robust.


\subsection{Robustness w.r.t. a perturbation}
We define a perturbation of a mass function \(\proba\) as a set of mass functions that contains \(\proba\).
\begin{definition}

	Consider a mass function \(\proba\) on \(\Classes \times \FeaturesSet\).
	Let \(\probfamily\) be a compact set of mass functions on \(\Classes \times \FeaturesSet\), with \(\proba \in \probfamily\).
	Then we call \(\probfamily\) a \emph{perturbation} of \(\proba\).
\end{definition}
This definition is quite general, as it allows for a wide range of perturbation types.
Usually, however, a perturbation \(\probfamily\) of \(\proba\) will be a `neighborhood' around $\proba$, consisting of mass functions of a particular type whose distance from $\proba$ is in some sense bounded.

There are several ways in which such a perturbation could be obtained.
The most straightforward one, which we will adopt here, is to directly perturb $\proba$ itself.
However, we could also perturb the data from which the model is learned and, in this way, indirectly perturb $\proba$ during the learning process; we leave this for future work.

Since a perturbation \(\probfamily\) is a set of mass functions, we can associate with each such mass function \(\proba'\in\probfamily\) a generative classifier \(\classifier_{\proba'}\)---or multiple ones, if there is no unique maximizer in Equation~\eqref{eq:prediction}.
If the prediction \(\classifier_{\proba'}(\features)\) of all those classifiers \(\classifier_{\proba'}\) is the same as the class $\prediction=\classifier_\proba(\features)$ that is predicted by  \(\classifier_{\proba}\), we say that the prediction of $\classifier_\proba$ is \emph{robust} w.r.t. the perturbation $\probfamily$.
\begin{definition}
	\label{def:robustness}
	Let \(\classifierpa\) be a generative classifier corresponding to a mass function \(\proba\).
	Let \(\probfamily\) a perturbation of \(\proba\) and let \(\prediction\) be the prediction according to \(\classifierpa\) for the set of features \(\features\).
	Then \(\prediction\) is \emph{robust} w.r.t. the perturbation \(\probfamily\) if \(\arg \max_{\class \in \Classes} \pertproba(c, \features)=\{\prediction\}\) for all \(\pertproba \in \probfamily\).
\end{definition}
Since \(\probfamily\) contains \(\proba\), this definition requires in particular that \(\arg \max_{\class \in \Classes} \proba(c, \features)=\{\prediction\}\), meaning that if the class \(\prediction\) predicted by the original classifier \(\classifier_\proba\) does not uniquely have the highest probability, then \(\prediction\) can never be robust.

To check if a prediction is robust w.r.t. a perturbation \(\probfamily\), we do not need to explicitly check for each individual $\pertproba\in\probfamily$ whether it uniquely predicts \(\prediction\); fortunately, we can instead reformulate robustness as a maximization problem.
This result---and our proof---is entirely analogous to \cite[Theorem~1]{NIPS2014_09662890}, where a similar reformulation was presented for the case without observed features.
\begin{theorem}
	\label{th:def_robustness}
	Let \(\classifierpa\) be a generative classifier corresponding to a mass function \(\proba\).
	Let \(\probfamily\) be a perturbation of \(\proba\) and let \(\prediction\) be the prediction according to \(\classifierpa\) for the set of features \(\features\).
	Then \(\prediction\) is \emph{robust} w.r.t. the perturbation \(\probfamily\) if and only if
	\begin{equation}
		\label{eq:th_def_robustness}
		\min_{\pertproba \in \probfamily} \pertproba(\prediction, \features) > 0 \quad \text{and} \quad \max_{\class \in \Classes \backslash \{\prediction\}} \max_{\pertproba \in \probfamily} \frac{\pertproba(\class, \features)}{\pertproba(\prediction, \features)} < 1,
	\end{equation}
	where the first inequality should be checked first because if it fails, the fraction in the second inequality is undefined.
\end{theorem}
\begin{proof}
	By definition, \(\prediction\) is robust w.r.t. \(\probfamily\) if and only if \(\arg \max_{\class \in \Classes} \pertproba(c, \features)=\{\prediction\}\) for all \(\pertproba \in \probfamily\).
	This is clearly the case if and only if
	\begin{equation}
		\label{eq:proof_robustness}
		\forall \pertproba \in \probfamily, \forall \class \in \Classes \backslash \prediction \colon \pertproba(\class, \features) < \pertproba(\prediction, \features).
	\end{equation}
	Since $\pertproba(\class, \features)\geq0$, this can only be true if \(\pertproba(\prediction, \features)>0\) for each \(\pertproba \in \probfamily\), or equivalently, due to the compactness of \(\probfamily\), if the first inequality of Equation \eqref{eq:th_def_robustness} holds.
	Assuming that it holds, Equation \eqref{eq:proof_robustness} can now be rewritten as
	\begin{equation}
		\max_{\class \in \Classes \backslash \{\prediction\}} \max_{\pertproba \in \probfamily} \frac{\pertproba(\class, \features)}{\pertproba(\prediction, \features)} < 1,
	\end{equation}
	where the compactness of \(\probfamily\) ensures that the maximum is well-defined.
\end{proof}

From the point of view of imprecise probabilities~\cite{augustin2014introduction}, our notion of robustness is closely related to credal---or imprecise---classification
\cite{ITIPclassification}.
Starting from a set of probabilities $\probfamily$---or a perturbation, in our language---this approach to classification provides a set-valued prediction
	$\cup_{\pertproba\in\probfamily}\arg\max_{c\in\Classes}\pertproba(\class, \features)$
that gathers the predictions $\classifier_{\pertproba}(\features)$ of every possible classifier $\classifier_\pertproba$ that corresponds to a joint mass function $\pertproba\in\probfamily$.
Robustness then corresponds to the specific situation where this (possibly) set-valued prediction consists of only one class, which is then necessarily equal to $\prediction$.

\subsection{Robustness metrics}\label{sec:robustnessmetrics}
To be able to numerically quantify the robustness of a prediction as the minimal perturbation for which it is no longer robust, we need to somehow express what the size of perturbation is.
To this end, we will make use of parametrized perturbations.
\begin{definition}
	\label{def:parametrized_pert}
	Consider a probability mass function \(\proba\) on \(\Classes\times\FeaturesSet\) and, for all \(\robmetric\in[0,1]\), a perturbation \(\probfamily_{\robmetric}\) of \(\proba\). Then the family $\bulletprobfamily\coloneq(\probfamily_{\robmetric})_{\robmetric\in[0,1]}$ is called a \emph{parametrized perturbation} of \(\proba\) if it satisfies the following conditions:
	\begin{itemize}
		\item if \(\robmetric = 0\), then \(\probfamily_{0}=\{\proba\}\);
		\item if \(\robmetric_1 < \robmetric_2\), then \(\probfamily_{\robmetric_1} \subset \probfamily_{\robmetric_2}\).
	\end{itemize}
\end{definition}
So the bigger the parameter \(\robmetric\), the bigger the perturbation.
The idea is now to increase \(\robmetric\) until the prediction is no longer robust w.r.t. \(\probfamily_{\robmetric}\), and to use the value of \(\robmetric\) for which this happens as a robustness metric.
\begin{definition}
	\label{def:robustness_metric}
	Let \(\classifierpa\) be a generative classifier corresponding to a mass function \(\proba\), and let \(\prediction\) be the prediction according to \(\classifierpa\) for the set of features \(\features\).
	Let \(\bulletprobfamily\) be a parametrized perturbation of \(\proba\).
	Then the \emph{robustness metric} $\robmetric_{\bulletprobfamily}(\features)$---if it exists---is the smallest value of $\robmetric$ for which \(\prediction\) is no longer robust w.r.t. \(\probfamily_{\robmetric}\).
\end{definition}

An important advantage of this approach, compared to both (epistemic) uncertainty quantification and credal classification, is that it avoids the problem of determining which perturbation \(\probfamily_{\robmetric}\)---or which $\robmetric\in[0,1]$---is `correct'.

That said, there are of course many ways to construct a parametrized perturbation \(\bulletprobfamily\), and hence many robustness metrics that can be obtained in this way.
The ones we will focus on are based on so-called \emph{\(\robmetric\)-contamination}~\cite{berger2013statistical}.
\begin{definition}\label{def:epsiloncontamination}
	Consider a probability mass function \(\proba\) on a set \(\mathcal{X}\), and let \(\Sigma_{\mathcal{X}}\) be the set of all possible probability mass functions on \(\mathcal{X}\).
	Then for any \(\robmetric \in [0,1]\), we define the \(\robmetric\)-contamination of \(\proba\) as the family of all mass functions that are convex mixtures, with mixture coefficient $\robmetric$, of $\proba$ and an arbitrary element of $\Sigma_{\mathcal{X}}$:
	\begin{equation}
		\label{eq:eps_contamination}
		\probmeasfamily{\proba,\,\robmetric} = \{(1-\robmetric)\proba + \robmetric\proba^* \colon \proba^* \in \Sigma_X\}.
	\end{equation}
\end{definition}

In particular, a first type of family of perturbations that we will consider, is the one obtained by directly applying $\robmetric$-contamination to the learned probability mass function $\proba$.
We call this family the \emph{global} parametrized perturbation of \(\proba\) and denote it by $\bulletprobfamilyPa^{\,\mathrm{glob}}$.
For all $\robmetric\in[0,1]$, the corresponding perturbation of $\proba$ is given by
\begin{equation}
	\label{eq:globpert}
	\globperta \coloneq \probmeasfamily{\proba,\,\robmetric}.
\end{equation}
We call the robustness metric that corresponds to $\bulletprobfamilyPa^{\,\mathrm{glob}}$ the \emph{global robustness metric} and, for notational convenience, denote it by $\globeps\coloneq\robmetric_{\smash{\bulletprobfamilyPa^{\,\mathrm{glob}}}}$. Our next result provides a closed-form expression.
\begin{theorem}
	\label{th:global_robustness}
	Consider a set of features \(\features\) and let \(\prediction\) be the prediction according to a generative classifier \(\classifierpa\) with joint probability mass function \(\proba\).
	Then
	\begin{equation}\label{eq:closedformforglob}
		\globeps(\features) = \frac{\proba(\prediction, \features) - \max_{\class \in \Classes \backslash \prediction}\proba(\class, \features)}{1 + \proba(\prediction, \features) -  \max_{\class \in \Classes \backslash \prediction}\proba(\class, \features)}.
	\end{equation}
\end{theorem}

\begin{proof}
If $\proba(\prediction, \features)=0$, it follows from Equation~\eqref{eq:joint_prediction} that $\proba(\class, \features)=0$ for all $\class\in\Classes$, which implies that $\prediction$ is not robust for $\smash{\globpertzeroa=\{\proba\}}$ because $\arg\max_{\class\in\Classes}\proba(\class,\features)=\Classes$ is not a singleton. So in this case, we should have that $\globeps(\features)=0$, which indeed corresponds to Equation~\eqref{eq:closedformforglob}.


This leaves us with the case $\proba(\prediction, \features)>0$. To show that~\eqref{eq:closedformforglob} is correct in that case too, we fix any $\robmetric\in[0,1)$.
Since $\proba(\prediction, \features)>0$, it follows from Equation~\eqref{eq:globpert} and Definition~\ref{def:epsiloncontamination} that $\pertproba(\prediction,\features)>0$ for all $\smash{\pertproba\in\globperta}$. Theorem \ref{th:def_robustness}, Equation~\eqref{eq:globpert} and Definition~\ref{def:epsiloncontamination} therefore imply that $\prediction$ is not robust w.r.t. $\globperta$ if and only if
\begin{equation}\label{eq:intermediateinequality}
	\max_{\class \in \Classes \backslash \{\prediction\}} \max_{\proba^* \in \Sigma_{\Classes \times \FeaturesSet}} \frac{(1-\robmetric)\proba(\class, \features)+\robmetric\proba^*(\class, \features)}{(1-\robmetric)\proba(\prediction, \features)+\robmetric\proba^*(\prediction, \features)}\geq1,
\end{equation}
where \(\Sigma_{\Classes \times \FeaturesSet}\) is the set of all mass functions on \(\Classes\times\FeaturesSet\). Since the inner maximum is clearly attained by the unique distribution $\proba^*\in\Sigma_{\Classes \times \FeaturesSet}$ that assigns probability one to $(\class,\features)$---and hence zero probability to $(\prediction,\features)$---the left-hand side of Equation~\eqref{eq:intermediateinequality} simplifies to
\begin{equation}
	\max_{\class \in \Classes \backslash \{\prediction\}} \frac{(1-\robmetric)\proba(\class, \features)+\robmetric}{(1-\robmetric)\proba(\prediction, \features)}
	= \frac{(1-\robmetric)\displaystyle \max_{\class \in \Classes \backslash \{\prediction\}} \proba(\class, \features)+\robmetric}{(1-\robmetric)\proba(\prediction, \features)}
\end{equation}
Rearranging the terms, it therefore follows that Equation~\eqref{eq:intermediateinequality} holds---or equivalently, that $\prediction$ is not robust w.r.t. $\globperta$---if and only if
\begin{equation}
	\frac{\proba(\prediction, \features) - \max_{\class \in \Classes \backslash \{\prediction\}} \proba(\class, \features)}{1 + \proba(\prediction, \features) - \max_{\class \in \Classes \backslash \{\prediction\}} \proba(\class, \features)}\leq\robmetric.
\end{equation}
Furthermore, since we know from Equation~\eqref{eq:joint_prediction} that \(0\leq\proba(\prediction, \features) - \max_{\class \in \Classes \backslash \prediction}\proba(\class, \features)\leq 1\), the left-hand side of this inequality is at most $\nicefrac{1}{2}$. The smallest value of $\varepsilon\in[0,1)$---and hence also of $\varepsilon\in[0,1]$---for which $\prediction$ is not robust w.r.t. $\globperta$ is therefore equal to this left-hand side. This is exactly the value given by Equation~\eqref{eq:closedformforglob}.
\end{proof}

It follows from this result that our global robustness metric also has an alternative, simple and perhaps more intuitive interpretation that does not make use of perturbations: as can be seen from Equation~\eqref{eq:closedformforglob}, \(\globeps\) turns out to be a monotone transformation of the difference $\proba(\prediction, \features) - \max_{\class \in \Classes \backslash \prediction}\proba(\class, \features)$ between the joint probabilities of the most likely and second most likely classes.


\section{Robustness quantification for the Naive Bayes Classifier}
\label{sec:naive_classification}

An important advantage of the parametrized perturbation $\smash{\globperta}$, and the corresponding robustness metric $\globeps$, is that they can be applied to any generative classifier. 
For specific types of classifiers, however, we can also consider more specific types of perturbations that are tailor-made, as we are about to illustrate for the simple case of the Naive Bayes Classifier~\cite{naiveBayes}.

\subsection{Naive Bayes Classifier}

The central assumption on which the Naive Bayes Classifier (NBC) is based is that the features are conditionally independent given the class. For all classes \(\class \in \Classes\) and features \(\features \in \FeaturesSet\), this implies that the joint probability mass function can be factorized as
\begin{equation}
	\label{eq:NB_joint}
	\proba(\class, \features) = \proba(\class) \prod_{i=1}^{\numfeatures} \proba(\feature{i}\vert\class),
\end{equation}
where \(\proba(\class) \coloneq \sum_{\features \in \FeaturesSet}\proba(\class, \features)\) is the marginal probability of $\class$ and \(\proba(\feature{i}\vert\class)\) is the conditional probability of the feature \(\feature{i}\) given \(\class\).
We collect all probability mass functions on $\Classes\times\FeaturesSet$ that satisfy Equation~\eqref{eq:NB_joint} in the set \(\smash{\Sigma_{\Classes \times \FeaturesSet}^{\mathrm{NBC}}}\).
Every mass function $p$ in \(\Sigma_{\Classes \times \FeaturesSet}^{\mathrm{NBC}}\) is completely determined by a local mass function \(
\proba_{\Class}\) on \(\Classes\), defined by \(\proba_{\Class}(\class) \coloneq \proba(\class)\) for all $\class\in\Classes$, and, for all $\class\in\Classes$, a local mass function \(\proba_{\Feature{i}\vert\class}\) on \(\FeatureSet{i}\), defined by \(\proba_{\Feature{i}\vert\class}(\feature{i}) \coloneq \proba(\feature{i}\vert\class)\) for all $\feature{i}\in\FeatureSet{i}$.

The main advantage of assuming independence of the features given the class is that we do not need to learn the joint probability mass function $\proba$ as a whole---which is often high-dimensional---but can focus on learning the---typically low-dimensional---local probability mass functions \(\proba_{\Class}\) and \(\proba_{\Feature{i}\vert\class}\) instead.
In the experiments in Section~\ref{sec:experiments}, the probabilities that make up these local mass functions are obtained as follows:
\begin{equation}
	\proba(\class) = \frac{n(\class) + \alpha}{n + \alpha \vert\Classes\vert}\ \ \text{ and }\ \ \proba(\feature{i}\vert\class) = \frac{n(\class, \feature{i}) + \alpha}{n(\class) + \alpha \vert\FeatureSet{i}\vert},
\end{equation}
where \(n\) is the total number of training instances, \(n(\class)\) is the number of instances with class \(\class\),
and \(n(\class, \feature{i})\) is the number of instances with class \(\class\) and feature value \(\feature{i}\).
These expressions correspond to the posterior predictive distributions of the Dirichlet-multinomial model~\cite{edition2013bayesian}.
For $\alpha=0$, the obtained probabilities are simply the observed relative frequencies of the different classes, and of the features given each of the classes. The addition of a small additive smoothing parameter $\alpha>0$ avoids that the denominator becomes zero---if not all combinations of classes and features are present in the data---and protects against overfitting.
The exact value of $\alpha$ is determined by optimizing it, using 5-fold cross validation on the training set.

The downside that comes with this advantage, on the other hand, is that the assumption that the features are conditionally independent given the class is often unrealistic---or `naive'---in practice.
Enforcing this assumption anyway will then make it impossible for $\prob$ to closely resemble $\trainprob$.
Nevertheless, and perhaps surprisingly, the performance of NBCs is often competitive with other, more complicated types of classifiers~\cite{Domingos1997,Friedman1997}.

\subsection{Two types of perturbations}
For the particular case of an NBC with mass function \(\proba\), we now consider two types of perturbations and their corresponding robustness metrics.

\paragraph{Global perturbations}
The first family of perturbations is the global parametrized perturbation $\bulletprobfamilyPa^{\,\mathrm{glob}}$ of Section~\ref{sec:robustnessmetrics}. The expression for the corresponding robustness metric $\globeps$ does not simplify further for the specific case of an NBC, and is still given by Equation~\eqref{eq:closedformforglob}, in which we can easily evaluate $\proba(\prediction, \features)$ and $\proba(\class, \features)$ with Equation~\eqref{eq:NB_joint}.



\paragraph{Local perturbations}
The second family of perturbations takes inspiration from the set of probabilities that defines a Naive Credal Classifier~\cite{ZAFFALON20025}---an imprecise-probabilistic generalization of an NBC.
It is tailor-made for NBCs, as it makes use of the fact that $\proba$ factorizes as in Equation~\eqref{eq:NB_joint}. In particular, for each $\robmetric\in[0,1]$, we let
\begin{equation}\label{eq:localperturbation}
	\locperta \coloneq \Big\{ \pertproba \in \Sigma_{\Classes \times \FeaturesSet}^{\mathrm{NBC}} \colon \pertproba_{\Class} \in \probmeasfamily{\proba_{\Class},\,\robmetric}, \pertproba_{\Feature{i}\vert\class} \in \probmeasfamily{\proba_{\Feature{i}\vert\class},\,\robmetric} \Big\}
\end{equation}
be the perturbation of $\proba$ obtained by separately perturbing each of local mass functions with $\robmetric$-contamination.\footnote{This is not the only possibility. We could for example also perturb each of the local mass functions of the NBC by a different amount \(\varepsilon_i\) and then let $\epsilon$ be the sum of these different amounts.}
We call the resulting family $\bulletprobfamilyPa^{\,\mathrm{loc}}$ the \emph{local} parametrized perturbation of \(\proba\) and call the corresponding robustness metric the \emph{local robustness metric}. For notational convenience, denote it by $\loceps\coloneq\robmetric_{\smash{\bulletprobfamilyPa^{\,\mathrm{loc}}}}$. Our next result gives a convenient characterization.
\begin{theorem}
	\label{th:local_robustness}
Consider a set of features \(\features\) and let \(\prediction\) be the prediction according to an NBC \(\classifierpa\) whose joint probability mass function \(\proba\) factorizes according to Equation~\eqref{eq:NB_joint}.
Then $\phi\colon[0,1)\to\mathbb{R}_{\geq0}$, defined for all $\robmetric\in[0,1)$ by
	\begin{equation}
		\phi(\varepsilon)\coloneq\max_{\class \in \Classes \backslash \{\prediction\}} \Big(\proba(\class) + \frac{\robmetric}{1-\robmetric} \Big) \prod_{i=1}^{\numfeatures} \Big(\proba(\feature{i} \vert \class) + \frac{\robmetric}{1-\robmetric} \Big),
	\end{equation}
	is a strictly increasing function and $\loceps(\features)$ is the unique value of $\robmetric$ for which $\phi(\robmetric)=\proba{(\prediction, \features)}$.
\end{theorem}

\begin{proof}
Since all the factors inside the maximum are nonnegative and strictly increasing in $\robmetric$, the same is clearly true for $\phi(\robmetric)$ itself. So it suffices to show that $\phi(\loceps(\features))=\proba{(\prediction, \features)}$.

If $\proba(\prediction, \features)=0$, it follows from Equation~\eqref{eq:joint_prediction} that $\proba(\class, \features)=0$ for all $\class\in\Classes$, which implies that $\prediction$ is not robust for $\smash{\locpertzeroa=\{\proba\}}$ because $\arg\max_{\class\in\Classes}\proba(\class,\features)=\Classes$ is not a singleton. So $\loceps(\features)=0$ and therefore
\begin{equation}
\phi(\loceps(\features))
=\phi(0)
=\max_{\class \in \Classes \backslash \{\prediction\}}\proba(\class,\features)=0=\proba{(\prediction, \features)},
\end{equation}
using Equation~\eqref{eq:NB_joint} for the second equality.

This leaves us with the case $\proba(\prediction, \features)>0$. To show that $\phi(\loceps(\features))=\proba{(\prediction, \features)}$ in that case too, we fix any $\robmetric\in[0,1)$.
Since $0<\proba(\prediction,\features)=\proba(\prediction)\prod_{i=1}^N\proba(\feature{i}\vert\prediction)$, it follows from Equation~\eqref{eq:localperturbation} and Definition~\ref{def:epsiloncontamination} that $\pertproba(\prediction,\features)>0$ for all $\smash{\pertproba\in\locperta}$. It therefore follows from Theorem \ref{th:def_robustness} that $\prediction$ is not robust w.r.t. $\locperta$ if and only if
	\begin{equation}\label{eq:intermediateinequality2}
		\max_{\class \in \Classes \backslash \{\prediction\}} \max_{\pertproba \in \locperta} \frac{\pertproba(\class, \features)}{\pertproba(\prediction, \features)} \geq 1.
	\end{equation}
	Ignoring the maximum over the classes for now, rewriting the innermost maximum with Equation~\eqref{eq:localperturbation} yields
	\begin{equation}
		\max_{\pertproba_{\Class}\in\probmeasfamily{\proba_{\Class},\,\robmetric}} \frac{\pertproba_{\Class}(\class)}{\pertproba_{\Class}(\prediction)}  \prod_{i=1}^{\numfeatures}
		\frac{\max_{\pertproba_{\Feature{i}\vert\class}\in\probmeasfamily{\proba_{\Feature{i}\vert\class},\,\robmetric}} \pertproba_{\Feature{i}\vert\class}(\feature{i})}{\min_{\pertproba_{\Feature{i}\vert\prediction}\in\probmeasfamily{\proba_{\Feature{i}\vert\prediction},\,\robmetric}} \pertproba_{\Feature{i}\vert\class}(\feature{i})}.
	\end{equation}
	For the first maximum in this expression, we find that
	\begin{align}
		\max_{\proba_{\Class}\in\probmeasfamily{\proba_{\Class},\,\robmetric}} \frac{\pertproba_{\Class}(\class)}{\pertproba_{\Class}(\prediction)}
		&= \max_{\proba^* \in \Sigma_{\Classes}} \frac{(1-\robmetric)\proba_{\Class}(\class) + \robmetric\proba^*(\class)}{(1-\robmetric)\proba_{\Class}(\prediction) + \robmetric\proba^*(\prediction)}\\
		&=\frac{(1-\robmetric)\proba(\class) + \robmetric}{(1-\robmetric)\proba(\prediction)},
	\end{align}
	 where the maximum was attained by the unique mass function $\proba^*\in\Sigma_{\Classes}$ that assigns probability one to $\class$---and hence zero probability to $\prediction$.

	 Using a similar line of reasoning, we find that the maxima in the product are equal to $(1-\varepsilon)\proba(\feature{i}\vert\class)+\varepsilon$, whereas the minima are equal to $(1-\varepsilon)\proba(\feature{i}\vert\prediction)$.

So we see that Equation~\eqref{eq:intermediateinequality2} holds if and only if
\begin{equation}
\max_{\class \in \Classes \backslash \{\prediction\}} \frac{\Big((1-\robmetric)\proba(\class)+\robmetric \Big) \prod_{i=1}^{\numfeatures} \Big((1-\robmetric)\proba(\feature{i} \vert \class) + \robmetric \Big)}{(1-\robmetric)\proba(\prediction)  \prod_{i=1}^{\numfeatures} (1-\robmetric)\proba(\feature{i} \vert \class)} \geq 1,
\end{equation}
or equivalently, $\phi(\robmetric)\geq\proba(\prediction)\prod_{i=1}^N\proba(\feature{i}\vert\prediction)=\proba(\prediction,\features)$, using Equation~\eqref{eq:NB_joint} for the last equality.

So, we see that for $\robmetric\in[0,1)$, $\prediction$ is not robust w.r.t. $\locperta$ if and only if $\phi(\robmetric)\geq\proba(\prediction,\features)$.

On the other hand, we already know that $\phi(\robmetric)$ is strictly increasing in $\robmetric$ and, since $\nicefrac{\robmetric}{(1-\robmetric)}\geq1$ for $\robmetric\geq\nicefrac{1}{2}$, we also know that $\phi(\nicefrac{1}{2})\geq1\geq\proba(\prediction,\features)$ for $\robmetric\geq\nicefrac{1}{2}$.

It therefore follows that there is a smallest $\robmetric\in[0,1]$ for which $\prediction$ is not robust w.r.t. $\locperta$, and that this smallest value---which is by definition equal to $\loceps(\features)$---is indeed the unique value of $\robmetric$ for which $\phi(\robmetric)=\proba{(\prediction, \features)}$.
\end{proof}

Due to this result, $\loceps(\features)$ can easily be evaluated in practice: it suffices to find the unique value of $\varepsilon\in[0,1]$ for which the strictly increasing function $\phi$ is equal to $\proba{(\prediction, \features)}$, for example with a simple bisection method.


\section{Experiments}
\label{sec:experiments}
To evaluate the performance of robustness quantification and compare it to uncertainty quantification, we use both methods to assess the reliability of the predictions issued by a simple Naive Bayes Classifier. The class has three possible values, and the four features can take 2, 3, 3 and 4 possible values, respectively.
To be able to study the effect of distribution shift and limited data, we generate our own data so we can control both aspects.

\subsection{The test set}
Our test distribution $\testprob$ is a mixture of two distributions: $\testprob = (1-\beta)\prob_{\mathrm{fix}} + \beta\prob_{\mathrm{rand}}$. In this mixture, $\prob_{\mathrm{fix}}$ is a fixed distribution that satisfies the Naive Bayes assumption (Equation~\eqref{eq:NB_joint}. The class probabilities are $0.4$, $0.35$, and $0.25$. For each class value $\class$ and feature $\Feature{i}$, we assign probability $0.85$ to one of the feature values and distribute the rest of the probability mass uniformly over the other feature values. This part of the mixture gives structure to $\testprob$, creating a correlation between the features and class that makes classification possible. $\prob_{\mathrm{rand}}$ is a randomly generated distribution, obtained by generating a random number for each $(\class,\features)$ and then normalizing so they sum to one. This part of the mixture makes sure that $\testprob$ does not satisfy the Naive Bayes assumption, and that the classification task is sufficiently difficult.
We don't analyze the effect of \(\beta\) in our experiments; it has a fixed value \(\beta=0.3\) throughout. The same is true for $\prob_{\mathrm{rand}}$: while it is random, we keep it fixed throughout all experiments. This implies that $\testprob$ is fixed as well.
The test set $\testset$ is a random sample of size \(N_{\mathrm{test}}=1000\) according to $\testprob$.

\subsection{The training sets}
Our training distributions are a mixture of $\testprob$ and a random distribution $\prob_{\mathrm{shift}}$, with \(\gamma \in [0,1]\) as mixture coefficient: $\trainprob = (1-\gamma)\testprob + \gamma\prob_{\mathrm{shift}}$. The maximum amount of distribution shift is determined by $\gamma$, but the actual amount of distribution shift also depends on the random distance between $\prob_{\mathrm{shift}}$ and $\testprob$. We use several values of $\gamma$ in our experiments to study the effect of distribution shift and consider multiple instantiations of $\prob_{\mathrm{shift}}$ to make sure that there is some actual distribution shift present.
For each $\trainprob$ obtained in this way, the corresponding training sets $\trainset$ are random samples of size \(N_{\mathrm{train}}\). To study the effect of limited data, our experiments consider several values of $N_{\mathrm{train}}$.

\subsection{A single experiment}\label{sec:singleexperiment}
For a single training set $\trainset$ of size $N_\mathrm{train}$, the experiment we conduct goes as follows. First, we use $\trainset$ to learn an NBC. Do to so, we begin by using 5-fold cross validation to optimize the smoothing parameter \(\alpha\) for maximal accuracy. Next, we use this optimal \(\alpha\) to learn a final NBC based on the whole training set. With the obtained NBC, for each of the 1000 instances in \(\testset\), we then make a prediction $\prediction$ based on the features $\features$, and compute the uncertainty metrics $\maxprob(\features)$, $\entropy(\features)$, $\aleatunc(\features)$, $\totunc(\features)$ and $\epistunc(\features)$ and robustness metrics $\globeps(\features)$ and $\loceps(\features)$ for that prediction.

\subsection{Accuracy-acceptance curves}\label{sec:curves}
To evaluate the performance of each of these metrics, for a single $\trainset$ and its corresponding NBC, we look at how good they are at assessing the reliability of the test instances.
To that end, we use each metric to order the 1000 instances in $\testset$.
For the uncertainty metrics, we do this in order of increasing uncertainty.
For the robustness metrics, we do this in order of decreasing robustness.
For each of these orderings, the hope is now of course that the first (least uncertain, most robust) instances are the most reliable ones, whereas the last (most uncertain, least robust) instances are the least reliable.
In order to evaluate this, we chose to use \emph{accuracy-acceptance curves}.\footnote{These are very similar to the so-called accuracy-rejection curves~\cite{nadeem2009accuracy} that are commonly used to compare classifiers with a reject option: the only difference is that we plot the accuracy as a function of the acceptance rate instead of the rejection rate. We prefer our version because it emphasizes that we are interested in the accuracy on the accepted (more reliable) instances.}
For each metric, such a curve plots the accuracy of the predictions of the NBC for the first $N$ instances in the ordering provided by the metric (so the accuracy for the $N$ most reliable instances, one would hope), as a function of the so-called acceptance rate $r=\nicefrac{N}{N_{\mathrm{test}}}=\nicefrac{N}{1000}$.
For $r=1$, this yields the accuracy on the entire test set, and hence the same result for all metrics.
For $r<1$, good metrics are able to increase this accuracy.
So better metrics yield higher accuracy-acceptance curves.

\subsection{Experimental setup}
We explore the influence of limited data and distribution shift by considering different sizes of training sets and different amounts of distribution shift.
The considered values for \(N_{\mathrm{train}}\) are \(\{25, 50, 100\}\) and the values for \(\gamma\) are \(\{0, 0.2, 0.4\}\).

For each combination of \(N_{\text{train}}\) and \(\gamma\) we generate 10 different training distributions \(\trainprob\) (with the same $\gamma$ and 10 different random $\prob_{\mathrm{shift}}$) and use each of these to sample 10 different training sets $\trainset$ of size $N_{\mathrm{train}}$.
For each of the resulting 100 training sets, we then run the experiment described in Section~\ref{sec:singleexperiment} and, for each of the considered metrics, construct an accuracy-acceptance curve as described in Section~\ref{sec:curves}.

To evaluate the overall performance of a metric, we consider the average of the 100 obtained accuracy-acceptance curves. To evaluate to variability of this performance (to which extent the 100 training sets yield different results), we consider the standard-deviation.
The results are depicted in Figure \ref{fig:combined_means} and \ref{fig:combined_stds}, respectively.

\begin{figure*}
	\centering
	\includegraphics[scale=1.25, trim=6.5cm 10.7cm 2.8cm 10cm]{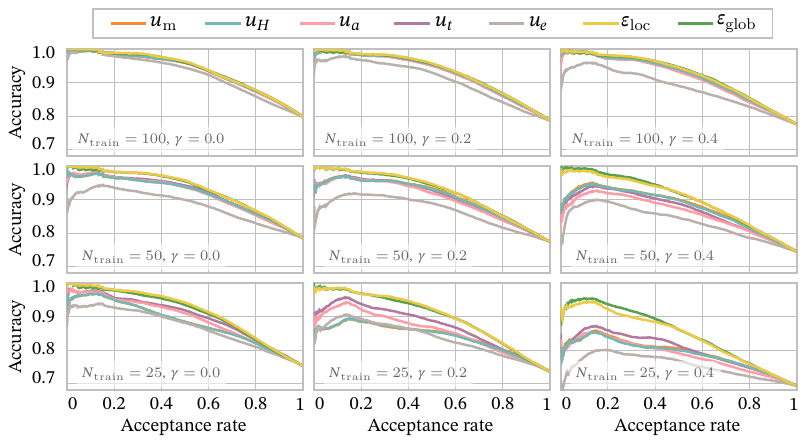}
	\caption{The means of the accuracy-acceptance curves for decreasing \(N_{\mathrm{train}}\) (100, 50, 35 from top to bottom) and increasing \(\gamma\) (0, 0.2, 0.4 left to right).}
	\label{fig:combined_means}
\end{figure*}

\begin{figure*}
	\centering
	\includegraphics[scale=1.25, trim=6.5cm 10.7cm 2.8cm 10cm]{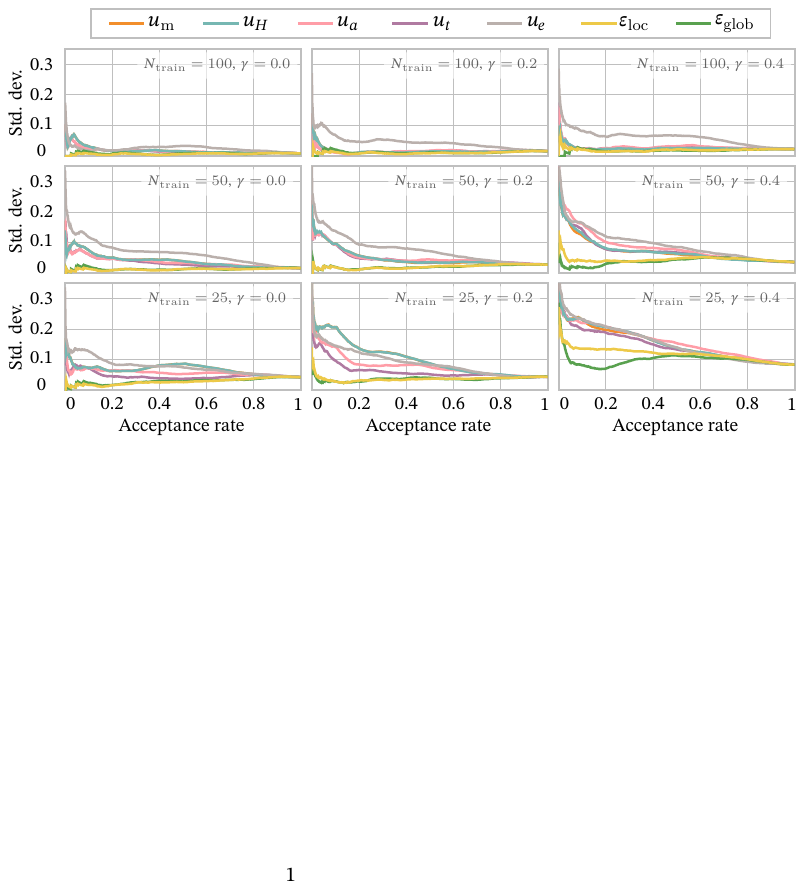}
	\caption{The standard deviations of the accuracy-acceptance curves for decreasing \(N_{\mathrm{train}}\) (100, 50, 25 from top to bottom) and increasing \(\gamma\) (0, 0.2, 0.4 left to right).}
	\label{fig:combined_stds}
\end{figure*}

\subsection{Results}
In Figure \ref{fig:combined_means}, we clearly see that the smaller the training set is, so comparing the plots from top to bottom, the worse the uncertainty metrics perform, since those curves get lower and lower, while the curves of the two robustness metrics stay relatively high.
A similar conclusion can be drawn for distribution shift. The more distribution shift there is, so if we move towards the right, the worse the performance of the uncertainty metrics is in comparison to the robustness measures, especially so for small training sets. The performance of the robustness metrics also seems to be the least affected by distribution shift or limited data, dropping considerably only for the most extreme case (on the bottom right). Meanwhile, even without distribution shift and for larger training sets, our robustness metrics are competitive with all the uncertainty metrics.
Comparing our two robustness metrics, we see that their performance is similar, with $\loceps$ performing slightly better for small data sets without distribution shift, and $\globeps$ performing better in the face of both challenges. 


Looking at Figure \ref{fig:combined_stds}, which displays the standard deviations, it becomes clear that our robustness metrics not only perform better in terms of overall average performance, but that their performance is also more stable, in the sense that it varies considerably less across the different training sets. 


\section{Conclusion/Discussion}
\label{sec:conclusion}

The main conclusion of our contribution is that the robustness metrics that are provided by robustness quantification do a good job at indicating to which extent a prediction of a generative classifier is reliable.
We furthermore see that robustness quantification is competitive with uncertainty quantification in scenarios with sufficiently large training sets and no distribution shift, and that it outperforms it in the face of both challenges.

So while robustness quantification ignores uncertainty, quantifying only how much uncertainty we could cope with if it were present, this seems to work better than quantifying epistemic uncertainty, or quantifying aleatoric uncertainty in the face of epistemic uncertainty.

Nevertheless, in our future work, since it seems to us that robustness and uncertainty metrics quantify entirely different aspects of reliability, we'd like to explore to which extent they can be combined to arrive at even better reliability metrics. To that end, we'd also first like to gain a better understanding about why, and in which contexts, robustness quantification works well.



Finally, we would of course like to confirm our conclusions with more extensive experiments, for classifiers with more complex (deep) model architectures and continous features, and for real classification problems based on collected rather than generated data.

\additionalinfo

\begin{acknowledgements}
We would like to thank the anonymous reviewers for their time, thorough reading and, especially, for their constructive and helpful feedback.
The work of both authors was partially supported by Ghent University's Special Research Fund, through Jasper De Bock's starting grant number 01N04819.
\end{acknowledgements}

\begin{authorcontributions}
Both authors contributed to the conceptualization of the ideas in this paper. The first author ran the experiments and wrote a first version of the paper, which then underwent a series of revisions under the supervision of the second author.
\end{authorcontributions}

\printbibliography

@article{Friedman1997,
  title    = {On Bias, Variance, 0/1—Loss, and the Curse-of-Dimensionality},
  journal  = {Data Mining and Knowledge Discovery},
  volume   = {1},
  pages    = {55-77},
  year     = {1997},
  doi      = {10.1023/A:1009778005914},
  author   = {Friedman, Jerome H.}
}

@article{Domingos1997,
  title    = {On the Optimality of the Simple Bayesian Classifier under Zero-One Loss},
  journal  = {Machine Learning},
  volume   = {29},
  pages    = {103-130},
  year     = {1997},
  doi      = {10.1023/A:1007413511361},
  author   = {Domingos, Pedro and Pazzani, Michael}
}

@book{naiveBayes,
  title     = {Pattern Classification and Scene Analysis},
  author    = {Duda, Richard O. and Hart, Peter E.},
  year      = {1973},
  publisher = {Wiley, New York},
  doi = {10.2307/2344977},
}

@inbook{ITIPclassification,
author = {Corani, Giorgio and Abellán, Joaquín and Masegosa, Andrés and Moral, Serafin and Zaffalon, Marco},
publisher = {John Wiley \& Sons, Ltd},
isbn = {9781118763117},
title = {Classification},
booktitle = {Introduction to Imprecise Probabilities},
chapter = {10},
pages = {230-257},
doi = {10.1002/9781118763117.ch10},
year = {2014}
}

@inproceedings{NIPS2014_09662890,
  author    = {De Bock, Jasper and de Campos, Cassio P. and Antonucci, Alessandro},
  booktitle = {Advances in Neural Information Processing Systems},
  pages     = {},
  publisher = {Curran Associates, Inc.},
  title     = {Global Sensitivity Analysis for MAP Inference in Graphical Models},
  url       = {https://proceedings.neurips.cc/paper_files/paper/2014/file/0966289037ad9846c5e994be2a91bafa-Paper.pdf},
  volume    = {27},
  year      = {2014},
}

@book{quinonero2022dataset,
  title     = {Dataset shift in machine learning},
  author    = {Qui{\~n}onero-Candela, Joaquin and Sugiyama, Masashi and Schwaighofer, Anton and Lawrence, Neil D.},
  year      = {2009},
  publisher = {The MIT Press},
  doi = {10.5555/1462129},
}

@article{ZAFFALON20025,
author = {Marco Zaffalon},
  title    = {The naive credal classifier},
  journal  = {Journal of Statistical Planning and Inference},
  volume   = {105},
  number   = {1},
  pages    = {5-21},
  year     = {2002},
  doi      = {10.1016/S0378-3758(01)00201-4}
}

@article{hullermeier2021aleatoric,
  title     = {Aleatoric and epistemic uncertainty in machine learning: An introduction to concepts and methods},
  author    = {H{\"u}llermeier, Eyke and Waegeman, Willem},
  journal   = {Machine learning},
  volume    = {110},
  number    = {3},
  pages     = {457--506},
  year      = {2021},
  publisher = {Springer},
  doi = {10.1007/s10994-021-05946-3},
}

@inproceedings{pmlr-v180-hullermeier22a,
  title     = {Quantification of Credal Uncertainty in Machine Learning: A Critical Analysis and Empirical Comparison},
  author    = {H\"ullermeier, Eyke and Destercke, S\'ebastien and Shaker, Mohammad Hossein},
  pages     = {548--557},
  year      = {2022},
  volume    = {180},
  series    = {Proceedings of Machine Learning Research},
  publisher = {PMLR},
  url       = {https://proceedings.mlr.press/v180/hullermeier22a.html},
}

@misc{correia2020robustclassificationdeepgenerative,
  title         = {Towards Robust Classification with Deep Generative Forests},
  author        = {Alvaro H. C. Correia and Robert Peharz and Cassio P. de Campos},
  year          = {2020},
  eprint        = {2007.05721},
  archiveprefix = {arXiv},
  primaryclass  = {stat.ML}
}

@inproceedings{pmlr-v62-mauá17a,
  title     = {Credal Sum-Product Networks},
  author    = {Mauá, Denis D. and Cozman, Fabio G. and Conaty, Diarmaid and Campos, Cassio P.},
  pages     = {205--216},
  year      = {2017},
  volume    = {62},
  series    = {Proceedings of Machine Learning Research},
  publisher = {PMLR},
  url       = {https://proceedings.mlr.press/v62/mau%C3%A117a.html},
}

@book{augustin2014introduction,
  title     = {Introduction to imprecise probabilities},
  author    = {Augustin, Thomas and Coolen, Frank P.A. and De Cooman, Gert and Troffaes, Matthias C.M.},
  year      = {2014},
  publisher = {John Wiley \& Sons},
  doi       = {10.1002/9781118763117}
}

@inproceedings{pmlr-v139-koh21a,
  title     = {WILDS: A Benchmark of in-the-Wild Distribution Shifts},
  author    = {Koh, Pang Wei and Sagawa, Shiori and Marklund, Henrik and Xie, Sang Michael and Zhang, Marvin and Balsubramani, Akshay and Hu, Weihua and Yasunaga, Michihiro and Phillips, Richard Lanas and Gao, Irena and Lee, Tony and David, Etienne and Stavness, Ian and Guo, Wei and Earnshaw, Berton and Haque, Imran and Beery, Sara M and Leskovec, Jure and Kundaje, Anshul and Pierson, Emma and Levine, Sergey and Finn, Chelsea and Liang, Percy},
  booktitle = {Proceedings of the 38th International Conference on Machine Learning},
  pages     = {5637--5664},
  year      = {2021},
  volume    = {139},
  series    = {Proceedings of Machine Learning Research},
  publisher = {PMLR},
  url       = {https://proceedings.mlr.press/v139/koh21a.html}
}

@article{zech2018variable,
  title     = {Variable generalization performance of a deep learning model to detect pneumonia in chest radiographs: a cross-sectional study},
  author    = {Zech, John R. and Badgeley, Marcus A and Liu, Manway and Costa, Anthony B and Titano, Joseph J and Oermann, Eric Karl},
  journal   = {PLoS medicine},
  volume    = {15},
  number    = {11},
  pages     = {1-17},
  year      = {2018},
  doi = {10.1371/journal.pmed.1002683},
}

@article{raudys1991small,
  title   = {Small sample size effects in statistical pattern recognition: Recommendations for practitioners},
  author  = {Raudys, Sarunas J. and Jain, Anil K. and others},
  journal={IEEE Transactions on Pattern Analysis and Machine Intelligence},
  volume  = {13},
  number  = {3},
  pages   = {252--264},
  year    = {1991},
  doi={10.1109/34.75512},
}

@article{vabalas2019machine,
  title     = {Machine learning algorithm validation with a limited sample size},
  author    = {Vabalas, Andrius and Gowen, Emma and Poliakoff, Ellen and Casson, Alexander J.},
  journal   = {PloS one},
  volume    = {14},
  number    = {11},
  pages     = {1-20},
  year      = {2019},
  doi = {10.1371/journal.pone.0224365},
}

@inproceedings{shaker2020aleatoric,
  title        = {Aleatoric and epistemic uncertainty with random forests},
  author       = {Shaker, Mohammad Hossein and H{\"u}llermeier, Eyke},
  booktitle    = {Advances in Intelligent Data Analysis XVIII},
  pages        = {444--456},
  year         = {2020},
  organization = {Springer},
  doi = {10.1007/978-3-030-44584-3_35},
}

@misc{sale2024label,
  title   = {Label-wise Aleatoric and Epistemic Uncertainty Quantification},
  author  = {Sale, Yusuf and Hofman, Paul and L{\"o}hr, Timo and Wimmer, Lisa and Nagler, Thomas and H{\"u}llermeier, Eyke},
  year    = {2024},
  eprint        = {2406.02354},
  archiveprefix = {arXiv}
}

@book{berger2013statistical,
  title     = {Statistical decision theory and Bayesian analysis},
  author    = {Berger, James O},
  year      = {2013},
  publisher = {Springer Science \& Business Media},
  doi       = {10.1007/978-1-4757-4286-2},
}

@book{edition2013bayesian,
  title     = {Bayesian Data Analysis},
  author    = {Gelman, A and Carlin, J.B. and Stern, H.S. and Dunson, D.B. and Vehtari, A. and Rubin, D.B.},
  year      = {2013},
  publisher = {CRC Press},
  doi       = {10.1201/b16018}
}

@inproceedings{nadeem2009accuracy,
  title        = {Accuracy-rejection curves (ARCs) for comparing classification methods with a reject option},
  author       = {Nadeem, Malik Sajjad Ahmed and Zucker, Jean-Daniel and Hanczar, Blaise},
  booktitle    = {Machine Learning in Systems Biology},
  pages        = {65--81},
  year         = {2009},
  organization = {PMLR},
}

@article{bai2021recent,
  title   = {Recent advances in adversarial training for adversarial robustness},
  author  = {Bai, Tao and Luo, Jinqi and Zhao, Jun and Wen, Bihan and Wang, Qian},
  journal = {arXiv preprint arXiv:2102.01356},
  year    = {2021},
}

@article{carlini2019evaluating,
  title   = {On evaluating adversarial robustness},
  author  = {Carlini, Nicholas and Athalye, Anish and Papernot, Nicolas and Brendel, Wieland and Rauber, Jonas and Tsipras, Dimitris and Goodfellow, Ian and Madry, Aleksander and Kurakin, Alexey},
  journal = {arXiv preprint arXiv:1902.06705},
  year    = {2019},
}

@inproceedings{NEURIPS2020_d8330f85,
  author    = {Taori, Rohan and Dave, Achal and Shankar, Vaishaal and Carlini, Nicholas and Recht, Benjamin and Schmidt, Ludwig},
  booktitle = {Advances in Neural Information Processing Systems},
  editor    = {H. Larochelle and M. Ranzato and R. Hadsell and M.F. Balcan and H. Lin},
  pages     = {18583--18599},
  publisher = {Curran Associates, Inc.},
  title     = {Measuring Robustness to Natural Distribution Shifts in Image Classification},
  url       = {https://proceedings.neurips.cc/paper_files/paper/2020/file/d8330f857a17c53d217014ee776bfd50-Paper.pdf},
  volume    = {33},
  year      = {2020},
}

\end{document}